\documentclass[letterpaper]{article} % DO NOT CHANGE THIS
\usepackage{aaai23}  % DO NOT CHANGE THIS
\usepackage{times}  % DO NOT CHANGE THIS
\usepackage{helvet}  % DO NOT CHANGE THIS
\usepackage{courier}  % DO NOT CHANGE THIS
\usepackage[hyphens]{url}  % DO NOT CHANGE THIS
\usepackage{graphicx} % DO NOT CHANGE THIS
\usepackage{natbib}  % DO NOT CHANGE THIS AND DO NOT ADD ANY OPTIONS TO IT
\usepackage{caption} % DO NOT CHANGE THIS AND DO NOT ADD ANY OPTIONS TO IT
\usepackage{algorithm}
\usepackage{algorithmic}

\usepackage{newfloat}
\usepackage{listings}
\DeclareCaptionStyle{ruled}{labelfont=normalfont,labelsep=colon,strut=off} % DO NOT CHANGE THIS
\floatstyle{ruled}
\newfloat{listing}{tb}{lst}{}
\floatname{listing}{Listing}
\title{Reward-Poisoning Attacks on Offline Multi-Agent Reinforcement Learning}

\author{
    Young Wu,
    Jeremy McMahan,
    Xiaojin Zhu, and
    Qiaomin Xie
}
\affiliations{
    \textsuperscript{\rm 1} University of Wisconsin-Madison\\
    yw@cs.wisc.edu, jmcmahan@wisc.edu, jerryzhu@cs.wisc.edu, qiaomin.xie@wisc.edu
}

\usepackage{bibentry}
\usepackage{amsmath,amssymb,amsthm}
\newtheorem{thm}{Theorem}
\newtheorem{cor}[thm]{Corollary}
\newtheorem{lem}[thm]{Lemma}
\newtheorem{prop}[thm]{Proposition}

\theoremstyle{definition}
\newtheorem{df}{Definition}

\newtheorem{asm}{Assumption}

\theoremstyle{remark}
\newtheorem{rmk}{Remark}

\begin{document}

%%Marco 
\global\long\def\mD{\mathcal{D}}%
\global\long\def\mDih{\mathcal{D}_{i,h}}%
\global\long\def\E{\mathbb{E}}%
\global\long\def\mA{\mathcal{A}}%
\global\long\def\mAi{\mathcal{A}_{i}}%
\global\long\def\mAni{\mathcal{A}_{-i}}%
\global\long\def\mM{\mathcal{M}}%
\global\long\def\mS{\mathcal{S}}%
\global\long\def\CIQ{\textup{CI}^{Q}}%
\global\long\def\CIV{\textup{CI}^{V}}%
\global\long\def\Nh{N_{h}}%
\global\long\def\NH{N_{H}}%
\global\long\def\Nl{\underline{N}}%
\global\long\def\Nlh{\underline{N}_{h}}%
\global\long\def\NlH{\underline{N}_{H}}%
\global\long\def\mG{G}%
\global\long\def\CIG{\textup{CI}^{G}}%
\global\long\def\Ckihs{C^{(k)}_{i,h,(s)}}%
\global\long\def\CkiHs{C^{(k)}_{i,H,(s)}}%
 
\global\long\def\aa{{\boldsymbol{a}}}% 
\global\long\def\ai{a_{i}}%
\global\long\def\ani{a_{-i}}%
\global\long\def\aih{a_{i,h}}%
\global\long\def\anih{a_{-i,h}}%
\global\long\def\ak{{\boldsymbol{a}}^{(k)}}%
\global\long\def\aki{a^{(k)}_{i}}%
\global\long\def\akni{a^{(k)}_{-i}}%
\global\long\def\akh{{\boldsymbol{a}}^{(k)}_{h}}%
\global\long\def\akih{a^{(k)}_{i,h}}%
\global\long\def\aknih{a^{(k)}_{-i,h}}%
\global\long\def\skh{s^{(k)}_{h}}%
\global\long\def\skhh{s^{(k)}_{h+1}}%
\global\long\def\sk{s^{(k)}}%
\global\long\def\akH{{\boldsymbol{a}}^{(k)}_{H}}%
\global\long\def\akiH{a^{(k)}_{i,H}}%
\global\long\def\akniH{a^{(k)}_{-i,H}}%
\global\long\def\skH{s^{(k)}_{H}}%
 
\global\long\def\pih{{\boldsymbol{\pi}}_{h}}%
\global\long\def\pihi{\pi_{i,h}}%
\global\long\def\pihni{\pi_{-i,h}}%
\global\long\def\pidag{{\boldsymbol{\pi}}^{\dagger}}%
\global\long\def\pidagi{\pi_{i}^{\dagger}}%
\global\long\def\pidagni{\pi_{-i}^{\dagger}}%
\global\long\def\pidagh{{\boldsymbol{\pi}}_{h}^{\dagger}}%
\global\long\def\pidaghh{{\boldsymbol{\pi}}_{h+1}^{\dagger}}%
\global\long\def\pidagih{\pi_{i,h}^{\dagger}}%
\global\long\def\pidagnih{\pi_{-i,h}^{\dagger}}%
\global\long\def\pidagihh{\pi_{i,h+1}^{\dagger}}%
\global\long\def\piH{{\boldsymbol{\pi}}_{H}}%
\global\long\def\piHi{\pi_{i,H}}%
\global\long\def\piHni{\pi_{-i,H}}%
\global\long\def\pidagH{{\boldsymbol{\pi}}_{H}^{\dagger}}%
\global\long\def\pidagiH{\pi_{i,H}^{\dagger}}%
\global\long\def\pidagniH{\pi_{-i,H}^{\dagger}}%
 
\global\long\def\ro{r^{0}}%
\global\long\def\rr{{\boldsymbol{r}}}%
\global\long\def\rk{{\boldsymbol{r}}^{0,(k)}}%
\global\long\def\rki{r_{i}^{0,(k)}}%
\global\long\def\rkh{{\boldsymbol{r}}_{h}^{0,(k)}}%
\global\long\def\rkih{r_{i,h}^{0,(k)}}%
\global\long\def\rdag{r^{\dagger}}%
\global\long\def\rdagi{r_{i}^{\dagger}}%
\global\long\def\rdagk{{\boldsymbol{r}}^{\dagger,(k)}}%
\global\long\def\rdagki{r_{i}^{\dagger,(k)}}%
\global\long\def\rdagkh{{\boldsymbol{r}}_{h}^{\dagger,(k)}}%
\global\long\def\rdagkih{r_{i,h}^{\dagger,(k)}}%
\global\long\def\rdagknih{r_{-i,h}^{\dagger,(k)}}%
\global\long\def\rkH{{\boldsymbol{r}}_{H}^{0,(k)}}%
\global\long\def\rkiH{r_{i,H}^{0,(k)}}%
\global\long\def\rdagkH{{\boldsymbol{r}}_{H}^{\dagger,(k)}}%
\global\long\def\rdagkiH{r_{i,H}^{\dagger,(k)}}%
\global\long\def\rdagkniH{r_{-i,H}^{\dagger,(k)}}%
 
\global\long\def\Rdag{R^{\dagger}}%
\global\long\def\Rdagi{R_{i}^{\dagger}}%
\global\long\def\Rdagh{{\boldsymbol{R}}_{h}^{\dagger}}%
\global\long\def\Rdagih{R_{i,h}^{\dagger}}%
\global\long\def\Ri{R_{i}}%
\global\long\def\Rh{{\boldsymbol{R}}_{h}}%
\global\long\def\Rih{R_{i,h}}%
\global\long\def\Rhih{\hat{R}_{i,h}}%
\global\long\def\RdagH{{\boldsymbol{R}}_{H}^{\dagger}}%
\global\long\def\RdagiH{R_{i,H}^{\dagger}}%
\global\long\def\RH{{\boldsymbol{R}}_{H}}%
\global\long\def\RiH{R_{i,H}}%
\global\long\def\RhiH{\hat{R}_{i,H}}%
\global\long\def\CIR{\textup{CI}^{R^{\dagger}}}%
\global\long\def\CIRi{\textup{CI}_{i}^{R^{\dagger}}}%
\global\long\def\CIRh{\textup{CI}_{h}^{R^{\dagger}}}%
\global\long\def\CIRih{\textup{CI}_{i,h}^{R^{\dagger}}}%
\global\long\def\CIRhih{\textup{CI}_{i,h}^{R}}%
\global\long\def\CIRH{\textup{CI}_{H}^{R^{\dagger}}}%
\global\long\def\CIRiH{\textup{CI}_{i,H}^{R^{\dagger}}}%
\global\long\def\CIRhiH{\textup{CI}_{i,H}^{R}}%
 
\global\long\def\Rl{\underline{R}^{\dagger}}%
\global\long\def\Ru{\overline{R}^{\dagger}}%
\global\long\def\Rlih{\underline{R}_{i,h}^{\dagger}}%
\global\long\def\Ruih{\overline{R}_{i,h}^{\dagger}}%
 
\global\long\def\Ph{P_{h}}%
\global\long\def\Phh{\hat{P}_{h}}%
\global\long\def\CIP{\textup{CI}^{P}}%
\global\long\def\CIPh{\textup{CI}_{h}^{P}}%
 
\global\long\def\rhoR{\rho^{R}}%
\global\long\def\rhoP{\rho^{P}}%
\global\long\def\rhoRh{\rho^{R}_{h}}%
\global\long\def\rhoRH{\rho^{R}_{H}}%
\global\long\def\rhoPh{\rho^{P}_{h}}%
 
\global\long\def\Vi{V_{i}}%
\global\long\def\Vh{{\boldsymbol{V}}_{h}}%
\global\long\def\Vhh{{\boldsymbol{V}}_{h+1}}%
\global\long\def\Vih{V_{i,h}}%
\global\long\def\ViH{V_{i,H}}%
\global\long\def\Viz{V_{i,0}}%
\global\long\def\Vihh{V_{i,h+1}}%
\global\long\def\ViHH{V_{i,H+1}}%
\global\long\def\Vl{\underline{V}}%
\global\long\def\Vu{\overline{V}}%
 
\global\long\def\Qh{{\boldsymbol{Q}}_{h}}%
\global\long\def\Qhh{{\boldsymbol{Q}}_{h+1}}%
\global\long\def\Qih{Q_{i,h}}%
\global\long\def\Qihh{Q_{i,h+1}}%
\global\long\def\Ql{\underline{Q}}%
\global\long\def\Qu{\overline{Q}}%
\global\long\def\Qli{{\boldsymbol{\underline{Q}}}_{i}}%
\global\long\def\Qui{{\boldsymbol{\overline{Q}}}_{i}}%
\global\long\def\Qlh{{\boldsymbol{\underline{Q}}}_{h}}%
\global\long\def\Quh{{\boldsymbol{\overline{Q}}}_{h}}%
\global\long\def\Qlih{\underline{Q}_{i,h}}%
\global\long\def\Quih{\overline{Q}_{i,h}}%
\global\long\def\Qlihh{\underline{Q}_{i,h+1}}%
\global\long\def\Quihh{\overline{Q}_{i,h+1}}%
\global\long\def\QH{{\boldsymbol{Q}}_{H}}%
\global\long\def\QiH{Q_{i,H}}%
\global\long\def\QlH{{\boldsymbol{\underline{Q}}}_{H}}%
\global\long\def\QuH{{\boldsymbol{\overline{Q}}}_{H}}%
\global\long\def\QliH{\underline{Q}_{i,H}}%
\global\long\def\QuiH{\overline{Q}_{i,H}}%
\global\long\def\LUCBQih{\textup{LUCB-}Q_{i,h}}%
\global\long\def\LUCBQ{\textup{LUCB-}Q}%
 
\global\long\def\uu{\overline{u}}%
\global\long\def\vu{\overline{v}}%
\global\long\def\wu{\overline{w}}%
\global\long\def\ul{\underline{u}}%
\global\long\def\vl{\underline{v}}%
\global\long\def\wl{\underline{w}}%
\global\long\def\uuih{\overline{u}_{i,h}}%
\global\long\def\vuih{\overline{v}_{i,h}}%
\global\long\def\wuih{\overline{w}_{i,h}}%
\global\long\def\ulih{\underline{u}_{i,h}}%
\global\long\def\vlih{\underline{v}_{i,h}}%
\global\long\def\wlih{\underline{w}_{i,h}}%
\global\long\def\tki{t_{i}^{(k)}}%
\global\long\def\tkih{t_{i,h}^{(k)}}%
\global\long\def\tkiH{t_{i,H}^{(k)}}%
\global\long\def\diih{d_{i,h}^{\iota}}%
\global\long\def\diiH{d_{i,H}^{\iota}}%

\global\long\def\mup{\overline{m}^{+}}%
\global\long\def\mum{\overline{m}^{-}}%
\global\long\def\mlp{\underline{m}^{+}}%
\global\long\def\mlm{\underline{m}^{-}}%

\maketitle
\raggedbottom
\allowdisplaybreaks
\begin{abstract}
In offline multi-agent reinforcement learning (MARL), agents estimate policies from a given dataset. We study reward-poisoning attacks in this setting where an exogenous attacker modifies the rewards in the dataset before the agents see the dataset. The attacker wants to guide each agent into a nefarious target policy while minimizing the $L^p$ norm of the reward modification.
Unlike attacks on single-agent RL, we show that the attacker can install the target policy as a Markov Perfect Dominant Strategy Equilibrium (MPDSE), which rational agents are guaranteed to follow. This attack can be significantly cheaper than separate single-agent attacks. We show that the attack works on various MARL agents including uncertainty-aware learners, and we exhibit linear programs to efficiently solve the attack problem. We also study the relationship between the structure of the datasets and the minimal attack cost. Our work paves the way for studying defense in offline MARL.
\end{abstract}

\section{Introduction}

%In most real-life situations, people’s decisions are influenced by the actions of others and the information they have from the past, experience. \jmcomment{Examples here, stocks, vaccines, etc}
%An obvious example is stock trading. Whether a stock trader will want to buy or trade a certain stock is based off of information they have gathered in the past in addition to how others feel about that stock. By changing the public’s opinion about a company, one could create a self fulfilling prophecy, also called feedback loops in economics. The public feels good about a company so invests more in it; this causes the company to have improved opportunities and become a better company. The cycle then continues.

%Beyond even decision making, human beliefs and opinions are usually crafted from their peers. Many people develop opinions or find information about a topic from their peers and those closest to them. They then share their reaffirming opinion back to the sources further building confidence in the opinion amongst the group members. This is clearly seen on social media. \jmcomment{More examples, volunteering, vaccines}
%The posts one sees highly influences their opinion, and then they propagate this opinion into their own posts. This frequently results in a phenomenon called echo chambers. On the positive side of things, often large scale collaboration of ideas lead to a wide-scale change of opinion for the public that leads to a good social outcome. One example of this is the world sharing information about covid and aligning on policies to minimize casualties.

Multi-agent reinforcement learning (MARL) has achieved tremendous empirical success across a variety of tasks such as the autonomous driving, cooperative robotics, economic policy-making, and video games. In MARL, several agents interact with each other and the underlying environment, and each of them aims to optimize their individual long-term reward \cite{zhang2021multi}. Such problems are often formulated under the framework of Markov Games \cite{shapley1953stochastic}, which generalizes the Markov Decision Process model from single-agent RL. 
In offline MARL, the agents aim to learn a good policy by exploiting a pre-collected dataset without further interactions with the environment or other agents~\cite{pan2022plan,jiang2021offline,cui2022offline_MARL,zhong2022offline_MARL}. The optimal solution in MARL typically involves equilibria concepts. 

%Here, agents choose actions based on the past history of actions taken by all agents and their corresponding outcomes. Then, the agents see the outcomes of their current actions, and must update their beliefs to make optimal future decisions. This process is formalized by a Markov Game. It is assumed that each agent is greedy so only optimizes their personal utility function. Thus, equilibria concepts from game theory are needed to rigorously reason about such systems.

While the above empirical success is encouraging, MARL algorithms are susceptible to data poisoning attacks: the agents can reach the wrong equilibria if an exogenous attacker manipulates the feedback to agents. For example, a third party attacker may want to interfere with traffic to cause autonomous vehicles to behave abnormally; teach robots an incorrect procedure so that they fail at certain tasks; misinform economic agents about the state of the economy and guide them to make irrational investment or saving decisions; or cause the non-player characters in a video game to behave improperly to benefit certain human players. In this paper, we study the security threat posed by reward-poisoning attacks on offline MARL.
Here, the attacker wants the agents to learn a target policy $\pi^\dagger$ of the attacker's choosing ($\pi^\dagger$ does not need to be an equilibrium in the original Markov Game). Meanwhile, the attacker wants to minimize the amount of dataset manipulation to avoid detection and accruing high cost.
This paper studies optimal offline MARL reward-poisoning attacks.
Our work serves as a first step toward eventual defense against reward-poisoning attacks.

%This paper studies optimal offline MARL reward-poisoning attacks.
%Our work serves as a first step toward eventual defense against reward-poisoning attacks.

%For a fixed type of subtlety constraint, we show that it is almost always possible to manipulate agents to achieve the outcome we desire. Changing feedback is often costly, so we also want to modify the game as little as possible. \jmcomment{More examples}
%For example, we might want to change the fewest number of social media posts to ensure a subgroup of people develop a certain opinion on a topic, say to break an echo chamber.

\subsection{Our Contributions}

We introduce reward-poisoning attacks in offline MARL.
We show that any attack that reduces to attacking single-agent RL separately must be suboptimal. Consequently, new innovations are necessary to attack effectively.
We present a reward-poisoning framework 
that guarantees the target policy $\pi^\dagger$ becomes a Markov Perfect Dominant Strategy Equilibrium (MPDSE) for the underlying Markov Game. 
Since any rational agent will follow an MPDSE if it exists, this ensures the agents adopt the target policy $\pi^\dagger$. 
We also show the attack can be efficiently constructed using a linear program.

The attack framework has several important features. First, it is effective against a large class of offline MARL learners rather than a specific learning algorithm. 
Second, the framework allows partially decentralized agents who can only access their own individual rewards rather than the joint reward vectors of all agents. 
Lastly, the framework only makes the minimal assumption on the rationality of the learners that they will not take dominated actions.

We also give interpretable bounds on the minimal cost to poison an arbitrary dataset. These bounds relate the minimal attack cost to the structure of the underlying Markov Game. Using these bounds, we derive classes of extremal games that are especially cheap or expensive for the attacker to poison. These results show which games may be more susceptible to an attacker, while also giving insight to the structure of multi-agent attacks.

In the right hands, our framework could be used by a benevolent entity to coordinate agents in a way that improves social welfare. However, a malicious attacker could exploit the framework to harm learners and only benefit themselves. Consequently, our work paves the way for future study of MARL defense algorithms. 

\subsection{Related Work}

\paragraph{Online Reward-Poisoning:}
Reward poisoning problem has been studied in various settings, including online single-agent reinforcement learners~\citep{banihashem2022admissible,  huang2019deceptive, liu2021provably, rakhsha2021policy, rakhsha2021reward, rakhsha2020policy, sun2020vulnerability, zhang2020adaptive}, as well as online bandits ~\citep{bogunovic2021stochastic, garcelon2020adversarial, guan2020robust, jun2018adversarial, liu2019data, lu2021stochastic, ma2018data, ming2020attack, zuo2020near}. Online reward poisoning for multiple learners is recently studied as a game redesign problem in ~\citep{ma2021game}. 

%In many real-world applications of multi-agent reinforcement learning, including autonomous driving, cooperative robots, traffic control, transportation control, board games or video games, and economic modelling, the multi-agent reinforcement learners do not have direct access to the environment and has to be trained on observational data or dataset generated from previous interactions with the environment. As a result, we focus on the offline setting in which the attacker modifies the rewards in the training dataset so it cannot influence the learners' decisions after their models are trained and deployed, in particular, the reward poisoning cannot be adaptive on the actions chosen by the learners, and as a result the attack cost can be pre-computed and optimized. We demonstrate the relative difficulty in offline attacks compared to online attacks in terms of the data coverage assumptions and the total attack costs measured by the amount of required reward modification.

\paragraph{Offline Reward Poisoning:} 
\citet{ma2019policy, rakhsha2020policy, rakhsha2021policy, rangiunderstanding, zhang2008value, zhang2009policy} focus on adversarial attack on offline single-agent reinforcement learners. 
\citet{gleave2019adversarial, guo2021adversarial} study the poisoning attack on multi-agent reinforcement learners, assuming that the attacker controls one of the learners. 
Our model instead assumes that the attacker is not one of the learners, and the attacker wants to and is able to poison the rewards of all learners at the same time. Our model pertains to  many applications such as autonomous driving, robotics, traffic control and economic analysis, in which there is a central controller whose interests are not aligned with any of the agents and can modify the rewards and therefore manipulate all agents at the same time. 

%These single-agent results do not directly apply to our multi-agent setting, in which the action of one learner affects the rewards and actions of all other learners and hence the learners cannot make decisions independently---common in applications in robotics, games, and economic systems. In fact, we show that formulating the problem as multiple separate single-agent reinforcement learning attacks leads to different poisoned rewards and are significantly more costly compared to the multi-agent reinforcement learning attacks. 

\paragraph{Constrained Mechanism Design:} 
Our paper is also related to the mechanism design literature, in particular, the K-implementation problem in~\citet{monderer2004k, anderson2010internal}. Our model differs mainly in that the attacker, unlike a mechanism designer, does not alter the game/environment directly, but instead modifies the training data, from which the learners infer the underlying game and compute their policy accordingly.
In practical applications, rewards are often stochastic due to imprecise measurement and state observation, hence the mechanism design approach is not directly applicable to MARL reward poisoning. Conversely, constrained mechanism design can be viewed a special case when the rewards are deterministic and the training data has uniform coverage of all period-state-action tuples.
%and hence our attack formulation as a linear program can be easily adapted to solve the corresponding mechanism design problem.

\paragraph{Defense against Attacks on Reinforcement Learning:} 
There is also recent work on defending against reward poisoning or adversarial attacks on reinforcement learning; examples include~\citet{banihashem2021defense, lykouris2021corruption, rangi2022saving, wei2022model, wu2022copa, zhang2021corruption, zhang2021robust}. These work focus on the single-agent setting where attackers have limited ability to modify the training data. We are not aware of defenses against reward poisoning in our offline multi-agent setting. Given the numerous real-world applications of offline MARL, we believe it is important to study the multi-agent version of the problem.

%\paragraph{Robust Multi-Agent Reinforcement Learning:}
%Our reward poisoning algorithm can successfully attack various types of learners, including naive offline learners who compute the optimal policy based on the maximum likelihood estimate of the Markov Game, as well as robust offline learners who leverage the principle of pessimism~\cite{cui2022offline_MARL, zhong2022offline_MARL}. 

%Mathematically, we formulate our attack on a class of learners who compute the policy based on estimated Q functions of Markov Game and who may be aware of the uncertainty due to insufficient coverage in the training dataset. We believe our attack could potentially work for other learners that do not directly compute the Q function.

\section{Preliminaries}
%\subsection{Offline MARL}

%The simplest generalization of Markov decision processes (MDPs) to the multi-agent setting are \emph{Markov Games}. 
\paragraph{Markov Games.}
 A finite-horizon general-sum $n$-player Markov Game is given by a tuple $\mG = (\mS, \mA, P, R, H, \mu)$ ~\citep{littman1994markov}. Here
$\mS$ is the finite state space, and $\mA=\mA_1\times\cdots\times\mA_n$ is the finite joint action space.
We use $\aa = (a_1,\ldots,a_n) \in \mA$ to represent a joint action of the $n$ learners; we sometimes write $\aa = (\ai, \ani)$ to emphasize that learner $i$ takes action $\ai$ and the other $n-1$ learners take joint action $\ani$.
For each period $h\in[H]$, $\Ph : \mS \times \mA \to  \Delta(\mS)$ is the transition function, where $\Delta(\mS)$ denotes the probability simplex on $\mS$, and $\Ph(s' | s, \aa)$ is the probability that the state is $s'$ in period $h + 1$ given the state is $s $ and the joint action is $\aa$ in period $h$. 
$\Rh : \mS \times \mA \to  \mathbb{R}^{n}$ is the mean reward function for the $n$ players, where $\Rih(s, \aa)$ denotes the scalar mean reward for player $i$ in state $s $ and period $h$ when the joint action $\aa$ is taken. The initial state distribution is $\mu$.

\paragraph{Policies and value functions.} We use $\pi$ to denote a deterministic Markovian \emph{policy} for the $n$ players, where $\pih : \mS \to  \mA$ is the policy in period $h$ and $\pih(s)$ specifies the joint action in state $s$ and period $h$. We write $\pih = (\pihi, \pihni)$, where $\pihi(s)$ is the action taken by learner $i$ and $\pihni(s)$ is the joint action taken by learners other than $i$ in state $s$ period $h$. 
The \emph{value} of a policy $\pi$ represents the expected cumulative rewards of the game assuming learners take actions according to $\pi$. Formally, the $Q$ value of learner $i$ in state $s$ in period $h$ under a joint action $\aa$ is given recursively by
 \begin{align*}
 \QiH^{\pi}\left(s, \aa\right) &= \RiH\left(s, \aa\right), \\ 
 \Qih^{\pi}\left(s, \aa\right) &= \Rih\left(s, \aa\right) 
 + \displaystyle\sum_{s' \in \mS} \Ph\left(s' | s, \aa \right) \Vihh^{\pi}\left(s'\right).
 \end{align*}
The value of learner $i$ in state $s$ in period $h$ under policy $\pi$ is given by
$\Vih^{\pi}\left(s\right) = \Qih^{\pi}\left(s, \pih\left(s\right)\right)$,
% \begin{equation}
% \ViH^{\pi}\left(s\right) = \RiH\left(s, \piH\left(s\right)\right); \;
% \Vih^{\pi}\left(s\right) = \Rih\left(s, \pih\left(s\right)\right) + \displaystyle\sum_{s' \in \mS} \Ph\left(s' | s, \pih\left(s\right)\right) \Vihh^{\pi}\left(s'\right), h \in \left[H-1\right].
% \end{equation}
% We use $\Vi = \sum_{s \in \mS} \mu\left(s\right) \Viz^{\pi}\left(s\right)$ to denote learner $i$'s value of the Markov Game.\qxcomment{Do we ever use $\Vi$?}
 and we use $\Vh^{\pi}(s) \in \mathbb{R}^{n}$ to denote the vector of values for all learners in state $s $ in period $h$ under policy $\pi$.

\paragraph{Offline MARL.} In offline MARL,
the learners are given a \emph{fixed} batch dataset $\mD$ that records historical plays of $n$ agents under some behavior policies, and no further sampling is allowed. We assume that $\mD = \big\{\big(\skh, \akh, \rkh\big)_{h=1}^{H}\big\}_{k=1}^{K}$ contains $K$ episodes of length $H$. The data tuple in period $h$ of episode $k$ consists of 
the state $\skh \in \mS$, the joint action profile $\akh \in \mA$, and reward vector $\rkh \in \mathbb{R}^{n}$,
where the superscript $0$ denotes the original rewards before any attack. The next state $\skhh$ can be found in the next tuple.%\footnote{For ease of exposition, we focus on episodic dataset in this paper. Our results can be generalized to other datasets such as those generated by iid sampling.} 
%\jzcomment{This is more restrictive than the usual definition of tuples that include the next state.  The latter allows the tuples to be scrambled or simply drawn from the occupancy distribution in an iid way, without needing to properly connect into trajectories.}
%\qxcomment{I think our results generalize to the less restrictive definition of tuples (Can Jeremey/Young confirm?), but many equations need to be rewritten for the general case. Do we want to add a footnote on this generalization?}
%$\mD$ is the training data shared by the $n$ learners. 
Given the shared data $\mD$, each learner independently constructs a  policy $\pi_i$ to maximize their own cumulative reward. 
They then behave according to the resulting joint policy ${\pi}=(\pi_1, \ldots, \pi_n)$ in future deployment. Note that in a multi-agent setting, the learners' optimal solution concept is typically an approximate Nash equilibrium or Dominant Strategy Equilibrium~\cite{cui2022offline_MARL,zhong2022offline_MARL}.

%where each episode in $\mD$ is of the form $\{(s_h, \aa_h, \rr_h,s_{h+1})\}_{h=1}^H$, where $s_h$ is the state at period $h$,  $\aa_h$ is the joint action vector of the $n$ agents, and $\rr_h$ is the reward vector for the $n$ agents.

An agent's access to $\mD$ may be limited, for example due to privacy reasons. There are multiple levels of accessibility. In the first level, the agents can only access data that directly involves itself: instead of the tuple $(s_h, \aa_h, \rr_h)$, agent~$i$ would only be able to see $(s_h, a_{i,h}, r_{i,h})$. 
In the second level, agent~$i$ can see the joint action but only its own reward: $(s_h, \aa_h, r_{i,h})$. 
In the third level, agent~$i$ can see the whole $(s_h, \aa_h, \rr_h)$. 
We focus on the second level in this paper.

Let $\Nh\left(s, \aa\right) = \sum_{k=1}^{K} \mathbf{1}_{\{\skh = s, \akh = \aa\}}$
be the total number of episodes containing $\left(s , \aa, \cdot \right)$ in period $h$. 
We consider a dataset $\mD$ that satisfies the following coverage assumption. 

\begin{asm} \label{asm:uca} 
(Full Coverage) For each $(s , \aa)$ and $h$, $\Nh\left(s, \aa\right) > 0$.
\end{asm}
%\qxcomment{The assumption is necessary for successful attack (attacker's perspective) or from learners' perspective? Do we have a Lemma or Proposition on this?}
While this assumption might appear strong, we later show that it is necessary to effectively poison the dataset.

%\qxcomment{Or allow the learners to verify the fake DSE?
%\footnote{The coverage condition can be relaxed to unilateral concentration~\citep{cui2022offline_MARL} or low relative uncertainty~\citep{zhong2022offline_MARL}, \emph{if} the attacker makes stronger assumption on learners rationality such that they will agree on a unique NE.}}

\subsection{Attack Model}

%\qxcomment{Below is an alternative version of the learner/attacker model. This model treats the confidence set as being constructed by the attackers, not the learners. The learners can use any algorithm, which may or may not explicitly use confidence sets. The only assumption is that the policy computed by the learners is a solution to some plausible game within the attacker's confidence set. (This is in response to the NeurIPS reviewer who complained that our attack requires knowing the learners' algorithms.)   }

We assume that the attacker has access to the original dataset $\mD$. The attacker has a pre-specified target policy $\pidag$ and attempts to poison the rewards in $\mD$ with the goal of forcing the learners to learn $\pidag$ from the poisoned dataset. The attacker also desires that the attack has minimal cost.
We let $C(\ro, \rdag)$ denote the cost of a specific poisoning, where $\ro = \big\{\big(\rkh\big)_{h=1}^{H}\big\}_{k=1}^{K}$ are the original rewards and $\rdag = \big\{\big(\rdagkh\big)_{h=1}^{H}\big\}_{k=1}^{K}$ are the poisoned rewards.
We focus on the $L^1$-norm cost $C(\ro, \rdag)=\|\ro-\rdag\|_1$.

\paragraph{Rationality.} For generality, the attacker makes minimal assumptions on the learners' rationality. Namely, the attacker only assumes that the learners never take dominated actions~\citep{monderer2004k}. For technical reasons, we strengthen this assumption slightly by introducing an arbitrarily small margin $\iota>0$ (e.g. representing the learners' numerical resolution).
 \begin{df} \label{df:mpdse} 
 A $\iota$-strict Markov perfect dominant strategy equilibrium ($\iota$-MPDSE) of a Markov Game $\mG$ is a policy $\pi$ satisfying that for all learners $i \in \left[n\right]$, periods $h \in \left[H\right]$, and states $s  \in \mS$,
 \begin{align*}
 \forall\, \ai \in \mAi, \ai \neq  \pihi(s), \ani &\in \mAni: \\
 \Qih^{\pi}\left(s, \left(\pihi(s), \ani\right)\right) &\ge  \Qih^{\pi}\left(s, \left(\ai, \ani\right)\right) + \iota .
 \end{align*}
 \end{df} 
Note that a strict MPDSE, if exists, must be unique.
\begin{asm}(Rationality)
The learners will play an $\iota$-MPDSE should one exist.
\end{asm}

\paragraph{Uncertainty-aware attack.} State-of-the-art MARL algorithms are typically uncertainty-aware \cite{cui2022offline_MARL,zhong2022offline_MARL}, meaning that learners are cognizant of the model uncertainty due to finite, random data and will calibrate their learning procedure accordingly. 
%In addition, each learner is aware of being in a game and their optimal solution concept is typically in the form of Nash equilibrium. 
The attacker accounts for such uncertainty-aware learners, but does not know the learners' specific algorithm or internal parameters. It only assumes that the policies computed by the learners are solutions to some game that is plausible given the dataset. Accordingly, the attacker aims to poison the dataset in such a way that the target policy is an  $\iota$-MPDSE for every game that is plausible for the poisoned dataset. 

To formally define the set of plausible Markov Games for a given dataset $\mD$, we first need a few definitions.
 \begin{df}(Confidence Game Set) \label{df:cbp} %\label{asm:cbr} 
 The confidence set on the transition function $\Ph\left(s, \aa\right)$ has the form:
 \begin{align*}
 \CIPh\left(s, \aa\right) := &\big\{\Ph\left(s, \aa\right) \in \Delta\left(\mA\right) : \\ &\|\Ph\left(s, \aa\right) - \Phh\left(s, \aa\right)\|_{1} \leq  \rhoPh\left(s, \aa\right)\big\} 
 \end{align*}
 where
 
 $\Phh(s' | s, \aa) := \dfrac{1}{\Nh(s, \aa)} \sum_{k=1}^{K} \mathbf{1}_{\{\skhh = s', \skh = s, \akh = \aa\}}$
 is the maximum likelihood estimate (MLE) of the true transition probability.
Similarly, the confidence set on the reward function $\Rih\left(s, \aa\right)$ has the form:
 \begin{align*}
 \CIRhih\left(s, \aa\right) := &\big\{\Rih\left(s, \aa\right) \in \left[-b, b\right]: \\
 & | \Rih\left(s, \aa\right) - \Rhih\left(s, \aa\right) | \leq  \rhoRh\left(s, \aa\right)\big\},
 \end{align*}
where
$\Rhih(s, \aa) := \dfrac{1}{\Nh(s, \aa)} \sum_{k=1}^{K} \rkih \mathbf{1}_{\{\skh = s, \akh = \aa\}}$
is the MLE of the reward. Then, the set of all plausible Markov Games consistent with $\mD$, denoted by $\CIG$, is defined to be:
 \begin{align*}
 \CIG := \big\{G = \left(\mS, \mA, P, R, H, \mu\right): \Ph\left(s, \aa\right) \in \CIPh\left(s, \aa\right), \\
 \Rih\left(s, \aa\right) \in \CIRhih\left(s, \aa\right), \forall\, i, h, s, \aa\big\}.
 \end{align*}
\end{df}

Note that both the attacker and the learners know that all of the rewards are bounded within $\left[-b, b\right]$ (we allow $b = \infty$). 
The values of $\rhoPh\left(s, \aa\right)$ and $\rhoRh\left(s, \aa\right)$ are typically given by concentration inequalities. One standard choice takes the Hoeffding-type form $\rhoPh\left(s, \aa\right) \propto 1/{\sqrt{\max\{N_h(s,a),1\}}},$ and $\rhoRh\left(s, \aa\right) \propto 1/{\sqrt{\max\{N_h(s,a),1\}}},$ where we recall that $N_h(s,a)$ is the visitation count of the state-action pair $(s,a)$ \cite{xie2020cce,cui2022offline_MARL,zhong2022offline_MARL}. We remark that with proper choice of $\rhoPh$ and $\rhoRh$, $\CIG$ contains the game constructed by optimistic MARL algorithms with upper confidence bounds \cite{xie2020cce}, as well as that by pessimistic algorithms with lower confidence bounds \cite{cui2022offline_MARL,zhong2022offline_MARL}. See the appendix for details. %Appendix~\ref{sec:compatible} for details. 

%The assumption on the confidence bounds of the rewards implies that the learners are aware of the bounds and thus clip the confidence bounds accordingly. 

With the above definition, we consider an attacker that attempts to modify the original dataset $\mD$ into $\mD^{\dagger}$ so that $\pidag$ is an $\iota$-MPDSE for every plausible game in $\CIG$ induced by the poisoned $\mD^{\dagger}$. This would guarantee the learners adopt~$\pidag$. 

The full coverage Assumption~\ref{asm:uca} is necessary for the above attack goal, as shown in the following proposition. We defer the proof to appendix.
\begin{prop} \label{prop:ucareq} 
If $\Nh\left(s, \aa\right) = 0$ for some $\left(h, s, \aa\right)$, then there exist MARL learners for which the attacker's problem is infeasible.
\end{prop}

\section{Poisoning Framework}

In this section, we first argue that naively applying single-agent poisoning attack separately to each agent results in suboptimal attack cost. We then present a new optimal poisoning framework that accounts for multiple agents and thereby allows for efficiently solving the attack problem. 

\paragraph{Suboptimality of single-agent attack reduction.}
As a first attempt, the attacker could try to use existing single-agent RL reward poisoning methods. However, this approach is doomed to be suboptimal. Consider the following game with $n=2$ learners, one period and one state:

%\begin{center}\footnotesize \begin{tabular}{|c|c|c|c|}
\begin{center} \begin{tabular}{|c|c|c|c|}
\hline
 $\mA_{1} \setminus  \mA_{2}$ &$1$ &$2$\\ \hline
$1$ &$\left(3, 3\right)$ &$\left(1, 2\right)$\\ \hline
$2$ &$\left(2, 1\right)$ &$\left(0, 0\right)$\\ \hline
\end{tabular} \end{center}

Suppose that the original dataset $\mD$ has full coverage. For simplicity, we assume that each $(s,\aa)$ pair appears sufficiently many times so that  $\rho^R$ is small. In this case, the target policy $\pidag = \left(1, 1\right)$ is already a MPDSE, so no reward modification is needed. However, if we use a single-agent approach, each learner $i$ will observe the following dataset:
%\begin{center}\footnotesize \begin{tabular}{|c|c|c|}
\begin{center} \begin{tabular}{|c|c|c|}
\hline
 $\mAi$ &$r$\\ \hline
$1$ &$\left\{3, 1\right\}$\\ \hline
$2$ &$\left\{2, 0\right\}$\\ \hline
\end{tabular} \end{center}
In this case, to learner $i$ it is not immediately clear which of the two actions is strictly better, for example, when ${1, 2}$ appears relatively more often then ${3, 0}$. To ensure that both players take action 1, the attacker needs to modify at least one of the rewards for each player, thus incurring a nonzero (and thus suboptimal) attack cost.

The example above shows that a new approach is needed to construct an optimal poisoning framework tailored to the multi-agent setting. Below we develop such a framework, first for the simple Bandit Game setting, which is then generalized to Markov Games.

\subsection{Bandit Game Setting}

As a stepping stone, we start with a subclass of Markov Games with $|\mS|=1$ and $H=1$, which are sometimes called bandit games. A bandit game consists of a single stage normal-form game. For now, we also pretend that the learners simply use the data to compute an MLE point estimate $\hat{\mG}$ of the game and then solve the estimated game $\hat{\mG}$. This is unrealistic, but it highlights the attacker's strategy to enforce that $\pidag$ is an $\iota$-strict DSE in $\hat{\mG}$. 

Suppose the original dataset is 
 $\mD = \big\{(\ak, \rk)\big\}_{k=1}^{K} $ (recall we no longer have state or period). Also, let $N(\aa) := \sum_{k=1}^{K} \mathbf{1}_{\{\ak = \aa\}}$ be the action counts. The attacker's problem can be formulated as a convex optimization problem given in~\eqref{eq:prb1}.%--\eqref{eq:prb1end}.

%\noindent\begin{minipage}{\linewidth}
%\small
\begin{equation}\label{eq:prb1}
 \begin{aligned}
 \displaystyle\min_{\rdag} \; &   C\big(\ro, \rdag\big) 
 \\ \text{\;s.t.\;} &  \Rdag(\aa) \!:=\! \dfrac{1}{N(\aa)} \displaystyle\sum_{k=1}^{K} \rdagk \mathbf{1}_{\left\{\ak = \aa\right\}},\forall\, \aa ;
 \\ &  \Rdagi\big(\pidagi, \ani\big) \geq  \Rdagi\left(\ai, \ani\right) + \iota, 
  \forall\; i, \ani, \ai 
 \neq  \pidagi ;
 \\ &  \rdagk \in \left[-b, b\right]^{n}, \forall\; k. %\label{eq:prb1end}
\end{aligned}
\end{equation}
%\end{minipage}
 %\hline
%  \begin{align} 
%  \displaystyle\min_{\rdag} \; &   C\left(\ro, \rdag\right) \label{eq:prb1}
%  \\ \text{\;s.t.\;} &  \Rdag\left(\aa\right) := \dfrac{1}{N\left(\aa\right)} \displaystyle\sum_{k=1}^{K} \rdagk \mathbf{1}_{\left\{\ak = \aa\right\}}, \forall\; \aa
%  \\ &  \Rdagi\left(\pidagi, \ani\right) \geq  \Rdagi\left(\ai, \ani\right) + \iota, \forall\; i, \ani, \ai 
%  \neq  \pidagi
%  \\ &  \rdagk \in \left[-b, b\right]^{n}, \forall\; k. \label{eq:prb1end}
%  \end{align}
The first constraint in~\eqref{eq:prb1} models the learners' MLE $\hat{\mG}$ after poisoning. The second constraint enforces that $\pidag$ is an $\iota$-strict DSE of $\hat{\mG}$ by definition.
We observe that: 
\begin{enumerate}
    \item The problem is feasible if $\iota \leq  2 b$, since the attacker can always set, for each agent, the reward to be $b$ for the target action  and  $-b$ for all other actions;
    \item If the cost function $C(\cdot,\cdot)$ is the $L^1$-norm, the problem is a linear program (LP) with $n K$ variables and $(A - 1) A^{n-1} + 2 n K$ inequality constraints (assuming each learner has $\left| \mAi \right| = A $ actions);
    \item After the attack, learner~$i$ only needs to see its own rewards to be convinced that $\pidagi$ is a dominant strategy; learner~$i$ does not need to observe other learners' rewards.  
\end{enumerate}
This simple formulation serves as an asymptotic approximation to the attack problem for confidence bound based learners. In particular, when $N(\aa)$ is large for all $\aa$, the confidence intervals on $P$ and $R$ are usually small. 
 
With the above idea in place, we can consider more realistic learners that are uncertainty-aware.
For these learners, the attacker attempts to enforce an $\iota$ separation between the lower bound of the target action's reward and the upper bounds of all other actions' rewards (similar to arm elimination in bandits). With such separation, all plausible games in $\CIG$ would have the target action profile as the dominant strategy equilibrium. This approach can be formulated as a slightly more complex optimization problem~\eqref{eq:prb2}, where the second and third constraints enforce the desired $\iota$ separation. The formulation~\eqref{eq:prb2} can be solved using standard optimization solvers, hence the optimal attack can be computed efficiently. 

%\noindent\begin{minipage}{\linewidth}
\begin{equation}\label{eq:prb2}
 \begin{aligned}
 \displaystyle\min_{\rdag} \; &   C(\ro, \rdag) 
 \\ \text{s.t.\;} &  \Rdag(\aa) := \dfrac{1}{N(a)} \displaystyle\sum_{k=1}^{K} \rdagk \mathbf{1}_{\left\{\ak = a\right\}}, \forall\; \aa ;
 \\ &  \CIRi(\aa) := \big\{\Ri(\aa) \in [-b, b] : \big| \Ri(\aa) - \Rdagi(\aa) \big| \\
 &\hspace{2cm}\leq  \rhoR(\aa)\big\}, \quad\; \forall\; i, \aa ;
 \\ &  \displaystyle\min_{\Ri \in \CIRi\!(\pidagi, \ani)} \! \Ri \geq \! \displaystyle\max_{\Ri \in \CIRi\!(\ai, \ani)} \Ri + \iota,  \\
 &\hspace{4cm} \forall\; i, \ani, \ai \neq  \pidagi ;
 \\ &  \rdagk \in \left[-b, b\right]^{n}, \forall\; k. %\label{eq:prb2end}
 \end{aligned}
 \end{equation}
%\end{minipage}

%  \begin{align}
%  \displaystyle\min_{\rdag} \; &   C\left(\ro, \rdag\right) \label{eq:prb2}
%  \\ \text{\;s.t.\;} &  \Rdag\left(\aa\right) := \dfrac{1}{N\left(a\right)} \displaystyle\sum_{k=1}^{K} \rdagk \mathbf{1}_{\left\{\ak = a\right\}}, \forall\; \aa
%  \\ &  \CIRi\left(\aa\right) := \left\{\Ri\left(\aa\right) \in \left[-b, b\right] : \left| \Ri\left(\aa\right) - \Rdagi\left(\aa\right) \right| \leq  \rhoRh\left(\aa\right)\right\}, \forall\; i, \aa
%  \\ &  \displaystyle\min_{\Ri \in \CIRi\left(\pidagi, \ani\right)} \Ri \geq  \displaystyle\max_{\Ri \in \CIRi\left(\ai, \ani\right)} \Ri + \iota, \forall\; i, \ani, \ai 
%  \neq  \pidagi
%  \\ &  \rdagk \in \left[-b, b\right]^{n}, \forall\; k. \label{eq:prb2end}
%  \end{align}

We next consider whether this formulation has a feasible solution. Below we characterize the feasibility of the attack in terms of the margin parameter $\iota$ and the confidence bounds.
\begin{prop} \label{prop:bfeas} 
The attacker's problem~\eqref{eq:prb2} is feasible if
%\begin{align}
$\iota \leq  2 b - 2 \rhoR\left(\aa\right), \forall\; \aa \in \mA.$
%\end{align}
\end{prop}
Proposition~\ref{prop:bfeas} is a special case of the general Theorem~\ref{thm:feas} with $H = \left| \mS \right| = 1$. %Appendix~\ref{proof: 4.2} as a special case of Theorem~\ref{thm:feas} when $H = \left| \mS \right| = 1$. 
We note that the condition in Proposition~\ref{prop:bfeas} has an equivalent form that relates to the structure of the dataset. We later present this form for more general case.

%When an $L^1$-norm cost function $C(\cdot,\cdot)$ is used, it is clear that \eqref{eq:prb1} and ~\eqref{eq:prb2} are linear programming problems; see appendix for details. %see Appendix~\ref{LP: 4.1.1} and Appendix~\ref{LP: 4.1.2}. 
When an $L^1$-norm cost function is used, we show in the appendix that the formulation~\eqref{eq:prb2} can also be efficiently solved.
\begin{prop} \label{prop:blp} 
With $L^1$-norm cost function $C\left(\cdot , \cdot \right)$, the problem~\eqref{eq:prb2} can be formulated as a linear program.
\end{prop}

%How costly it is for the attacker to poison bandit games?  We now give a worst-case lower bound on poison cost.
%\jzcomment{proposition 3: worst case analysis}

\subsection{Markov Game Setting}

We now generalize the ideas from the bandit setting to derive a poisoning framework for arbitrary Markov Games. With multiple states and periods, there are two main complications:
\begin{enumerate}
    \item In each period $h$, the learners' decision depends on $Q_h$, which involves both the immediate reward $R_h$ and the future return $Q_{h+1}$;
    \item The uncertainty in $Q_h$ amplifies as it propagates backward in $h$.
\end{enumerate}
Accordingly, the attacker needs to design the poisoning attack recursively. 

Our main technical innovation is an attack formulation based on \emph{$Q$ confidence-bound backward induction}. The attacker maintains confidence upper and lower bounds on the learners' $Q$ function, $\Qu$ and $\Ql$, with backward induction. To ensure $\pidag$ becomes an $\iota$-MPDSE, the attacker again attempts to $\iota$-separate the lower bound of the target action and the upper bound of all other actions, at all states and periods.

%  Given the training data
% $\mD := \Big\{\big(\akh, \skh, \rkh\big)_{h=1}^{H}\Big\}_{k=1}^{K}$
%  we define counts
% $\Nh(s, \aa) := \displaystyle\sum_{k=1}^{K} \mathbf{1}_{\big\{\skh = s, \akh = \aa\big\}}$,
% MLE transitions
% $\Phh(s' | s, \aa) := \dfrac{1}{\Nh(s, \aa)} \displaystyle\sum_{k=1}^{K} \mathbf{1}_{\big\{\skhh = s', \skh = s, \akh = \aa\big\}}$,
% and transition confidence sets
% $\CIPh(s, \aa) := \left\{\Ph(s, \aa) \in \Delta(\mA) : \big\|\Ph(s, \aa) - \Phh(s, \aa)\big\|_{1} \leq  \rhoPh(s, \aa)\right\}$.
 
Recall Definition~\ref{df:cbp}: %Assumption~\ref{asm:cbr} 
given the training dataset $\mD$, one can compute the MLEs $\Rh$ and corresponding confidence sets $\CIRhih$ for the reward. The attacker aims to poison $\mD$ into $\mD^{\dagger}$ so that the MLEs and confidence sets become $\Rdagh$ and $\CIRih$, under which $\pidag$ is the unique $\iota$-MPDSE for all plausible games in the corresponding confidence game set. The attacker finds the minimum cost way of doing so by solving a $Q$ confidence-bound backward induction optimization problem, given in \eqref{eq:prb4}--\eqref{eq:prb4end}.
%\begin{figure*}[t]
 \begin{align}
 \displaystyle\min_{\rdag} \; &   C\left(\ro, \rdag\right) \label{eq:prb4} 
 \\ \text{\;s.t.\;} &  \Rdagih\left(s, \aa\right) := \dfrac{1}{\Nh\left(s, \aa\right)} \displaystyle\sum_{k=1}^{K} \rdagkih \mathbf{1}_{\left\{\skh = s, \akh = \aa\right\}}, \nonumber
 \\ & \hspace{1 em} \forall\; h, s, i, \aa \nonumber
 \\ &  \CIRih\left(s, \aa\right) := \Big\{\Rih\left(s, \aa\right) \in \left[-b, b\right]  \nonumber
 \\ & \hspace{1 em} :  \nonumber
 \big| \Rih\left(s, \aa\right) - \Rdagih\left(s, \aa\right) \big| \leq  \rhoRh\left(s, \aa\right)\Big\}, \nonumber
 \\ & \hspace{1 em} \forall\; h, s, i, \aa \nonumber
 \\ &  \QliH\left(s, \aa\right) := \displaystyle\min_{\RiH \in \CIRiH\left(s, \aa\right)} \RiH, \forall\; s, i, \aa \nonumber
 \\ &  \Qlih\left(s, \aa\right) := \displaystyle\min_{\Rih \in \CIRih\left(s, \aa\right)} \Rih  \nonumber
 \\ & \hspace{1 em} + \displaystyle\min_{\Ph \in \CIPh\left(s, \aa\right)} \displaystyle\sum_{s' \in \mS} \Ph\left(s'\right) \Qlihh\left(s', \pidaghh\left(s'\right)\right),  \nonumber
 \\ & \hspace{1 em} \forall\; h < H, s, i, \aa \label{eq:prb4Ql}
 \\ &  \QuiH\left(s, \aa\right) := \displaystyle\max_{\RiH \in \CIRiH\left(s, \aa\right)} \RiH, \forall\; s, i, \aa \nonumber
 \\ &  \Quih\left(s, \aa\right) := \displaystyle\max_{\Rih \in \CIRih\left(s, \aa\right)} \Rih  \nonumber
 \\ & \hspace{1 em} + \displaystyle\max_{\Ph \in \CIPh\left(s, \aa\right)} \displaystyle\sum_{s' \in \mS} \Ph\left(s'\right) \Quihh\left(s', \pidaghh\left(s'\right)\right),  \nonumber
 \\ & \hspace{1 em} \forall\; h < H, s, i, \aa \label{eq:prb4Qu}
 \\ &  \Qlih\left(s, \big(\pidagih(s), \ani\big)\right) \geq  \Quih\left(s, \left(\ai, \ani\right)\right) + \iota, \nonumber
 \\ & \hspace{1 em} \forall\; h, s, i, \ani, \ai \label{eq:prb4margin}
 \neq  \pidagih\left(s\right)
 \\ &  \rdagkh \in \left[-b, b\right]^{n}, \forall\; h, k. \label{eq:prb4end}
 \end{align}

The backward induction steps~\eqref{eq:prb4Ql} and~\eqref{eq:prb4Qu} ensure that 
$\Ql$ and $\Qu$ are valid lower and upper bounds for the $Q$ function for all plausible Markov Games in $\CIG$, for all periods.
The margin constraints~\eqref{eq:prb4margin} enforces an $\iota$-separation between the target action and other actions at all states and periods. We emphasize that the agents need not consider $Q$ at all in their learning algorithm; $Q$ only appears in the optimization due to its presence in the definition of MPDSE.

Again, pairing an efficient optimization solver with the above formulation gives an efficient algorithm for constructing the poisoning. We now answer the important questions of whether this formulation admits a feasible solution and whether these solutions yield successful attacks. The lemma below provides a positive answer to the second question.
\begin{lem} \label{lem:cigdse} 
If the attack formulation \eqref{eq:prb4}--\eqref{eq:prb4end} is feasible, $\pidag$ is the unique $\iota$-MPDSE of every Markov Game $\mG \in \CIG$.
\end{lem}

Moreover, the attack formulation admits feasible solutions under mild conditions on the dataset. 
 \begin{thm} \label{thm:feas} 
 The attacker formulation~\eqref{eq:prb4}--\eqref{eq:prb4end}  is feasible if the following condition holds:
 \begin{align*}\label{eq:feasibility_cond}
 \iota \leq  2 b - \left(H + 1\right) \rhoRh\left(s, \aa\right), \;\; \forall\; h \in \left[H\right], s \in \mS, \aa \in \mA.
 %\vspace{-0.2in}
 \end{align*}
\end{thm}

%\noindent The proof is given in the Appendix~\ref{proof: 4.2}.
We remark that the learners know the upper bound $b$ and may use it exclude implausible games. The accumulation of confidence intervals over the $H$ periods results in the extra factor $(H+1)$ on $\rhoRh$. Theorem~\ref{thm:feas} implies that the problem is feasible so long as the dataset is sufficiently populated; that is, each $(s,a)$ pair should appear frequently enough to have a small confidence interval half-width $\rhoRh$. The following corollary provides a precise condition on the visit accounts that guarantees feasibility.
%Wtly, $b$ is not accumulated over the $H$ periods whereas $\rhoRh$ is accumulated over the $H$ periods. \qxcomment{The explanation is not very clear.}

\begin{cor} \label{cor:datareq} 
Given a confidence probability $\delta$ and the confidence interval half-width $\rhoRh\left(s, \aa\right) = f(\frac{1}{N_h(s,a)})$ for some strictly increasing function $f$, the condition in Theorem~\ref{thm:feas} holds if
\begin{align*}
\Nh(s, \aa) &\geq  \Big(f^{-1}\big(\frac{2b - \iota}{H+1}\big)\Big)^{-1}.
\end{align*}
In particular, for the natural choice of Hoeffding-type $\rhoRh\left(s, \aa\right) = 2b \sqrt{\dfrac{ \log\left(\left(H \left| \mS \right| \left| \mA \right|\right) / \delta\right)}{\displaystyle\max\left\{\Nh\left(s, \aa\right), 1\right\}}}$, 
%$\rhoRh\left(s, \aa\right) = 4 \sqrt{\dfrac{H^{2} \log\left(\left(H \left| \mS \right| \left| \mA \right|\right) / \delta\right)}{\displaystyle\max\left\{\Nh\left(s, \aa\right), 1\right\}}}$, 
it suffices that,
\begin{align*}
\Nh(s, \aa) &\geq  \dfrac{4 b^{2} \left(H + 1\right)^{2} \log\left(\left(H \left| \mS \right| \left| \mA \right|\right) / \delta\right)}{\left(2 b - \iota\right)^{2}}.
%\Nh\left(s, \aa\right) &\geq  \dfrac{16 H^{2} \left(H + 1\right)^{2} \log\left(\left(H \left| \mS \right| \left| \mA \right|\right) / \delta\right)}{\left(2 b - \iota\right)^{2}}.
\end{align*}
\end{cor}

Despite the inner min and max in the  problem~\eqref{eq:prb4}--\eqref{eq:prb4end}, the problem can  be formulated as an LP, thanks to LP duality. 
\begin{thm}\label{thm:LP 4.1.2}
With $L^1$-norm cost function $C(\cdot,\cdot)$, problem~\eqref{eq:prb4}--\eqref{eq:prb4end} can be formulated as an LP.% linear program.
\end{thm}

\noindent The proofs of the above results can be found in appendix. %Appendix~\ref{LP: 4.1.2}.

\section{Cost Analysis}

Now that we know how the attacker can poison the dataset in the multi-agent setting, we can study the structure of attacks. The structure is most easily seen by analyzing the minimal attack cost. To this end, we give general bounds that relate the minimal attack cost to the structure of the underlying Markov Game. The attack cost upper bounds show which games are particularly susceptible to poison, and the attack cost lower bounds demonstrate that some games are expensive to poison. 

\textbf{Overview of results:} Specifically, we shall present two types of upper/lower bounds on the attack cost: (i) \emph{universal bounds} that hold for all attack problem instances simultaneously; (ii)  \emph{instance-dependent bounds} that are stated in terms of certain properties of the instance. We also discuss problem instances under which these two types of bounds are tight and coincide with each other.

%We note that nearly every bound in this section can be refined. However, we choose to give slightly weaker bounds at times to maximize readability and better highlight the structural aspects of the attack. 
We note that all bounds presented here are with respect to the $L^1$-cost, but many of them generalize to other cost functions, especially the $L^{\infty}$-cost. The proofs of the results presented in this section are provided in the appendix.

\textbf{Setup:} Let $I = (\mD, \pidag, \rhoR, \rhoP, \iota)$ denote an instance of the attack problem, and $\hat{G}$ denote the corresponding MLE of the Markov Game derived from $\mD$. We denote by $I_h = (\mD_h, \pidagh, \rhoRh, \rhoPh, \iota)$ the restriction of the instance to period $h$. In particular, $\hat{R}_h(s)$ derived from $\mD_h$ is exactly the normal-form game at state $s$ and period $h$ of $\hat{G}$. We define $C^*(I)$ to be the optimal $L^1$-poisoning cost for the instance $I$; that is, $C^*(I)$ is the optimal value of the optimization problem \eqref{eq:prb4}--\eqref{eq:prb4end} evaluated on $I$. We say the attack instance $I$ is \emph{feasible} if this optimization problem  is feasible. If $I$ is infeasible, we define $C^*(I) = \infty$. WLOG, we assume that $|\mA_1| = \cdots = |\mA_n| = A$. In addition, we define the minimum visit count for each period $h$ in $\mD$ as $\Nlh := \min_{s \in \mS} \min_{\aa \in \mA} \Nh\left(s, \aa\right)$, and the minimal over all periods as $\Nl := \min_{h \in H} \Nlh$. We similarly define the maximum visit counts as $\overline{N}_h = \max_{s \in \mS} \max_{\aa \in \mA} \Nh\left(s, \aa\right)$ and $\overline{N} = \max_h \overline{N}_h$. Lastly, we define $\underline{\rho} = \min_{h,s,\aa} \rhoRh(s,\aa)$ and $\overline{\rho} = \max_{h,s,\aa} \rhoRh(s,\aa)$, the minimum and maximum confidence half-width.

\subsection{Universal Cost Bounds}

With the above definitions, we present universal attack cost bounds that hold simultaneously for all attack instances.
\begin{thm}\label{thm:univeral_cost_bounds}
For any feasible attack instance $I$, we have that,
\[0 \leq C^*(I) \leq  \overline{N}H|\mS|nA^n2b.\]
\end{thm}
%\overline{N} H|\mS|nA^{n-1}(2b + 2\overline{\rho} + \iota)
\noindent As these upper and lower bounds hold for all instances, they are typically loose. However, they are nearly tight. If $\pidag$ is already an $\iota$-MPDSE for all plausible games, then no change to the rewards is needed and the attack cost is $0$, hence the lower bound is tight for such instances. We can also construct a high cost instance to show near-tightness of the upper bound. 

Specifically, consider the dataset for a bandit game, $\mD = \big\{(\ak, \rk)\big\}_{k=1}^{K},$ where $\mathcal{A} = A^n$ and each action appears exactly $N$ times, i.e., $\overline{N} = \underline{N} = N$ and  $K = N A^n$. The target policy is $\pidag = (1,\ldots, 1)$. The dataset is constructed so that $\rk_i = -b$ if $\ak_i = \pidagih\left(s\right) $ and $\rk_i = b$ otherwise. These rewards are essentially the extreme opposite of what the attacker needs to ensure $\pidag$ is an $\iota$-DSE. Note, the dataset induces the MLE of the game shown in Table~\ref{table:MLE_before} for the special case with $n = 2$ players.
\begin{center}
%\footnotesize
%\setlength\tabcolsep{3pt}
\begin{tabular}{|c|c|c|c|c|c|}
 \hline
 $\mA_{1} / \mA_{2}$ &$1$ &$2$ &$...$ &$\left| \mA_{2} \right|$\\ \hline
 $1$ &$-b, -b$ &$-b, b$ &$...$ &$-b, b$\\ \hline
 $2$ &$b, -b$ &$b, b$ &$...$ &$b, b$\\ \hline
 $...$ &$...$ &$...$ &$...$ &$...$\\ \hline
 $\left| \mA_{1} \right|$ &$b, -b$ &$b, b$ &$...$ &$b, b$\\ \hline
 \end{tabular}
\captionof{table}{ MLE $\hat{\Rh}(s,\cdot)$ before attack} \label{table:MLE_before}
\end{center}
%for all $ s \in \mS, \aa \in \mA.$
%  \begin{align}
%  \Rhih\left(s, a\right) &= \begin{cases} -b & \text{\;if\;} \ai = \pidagih\left(s\right) \\ b & \text{\;otherwise\;} \\ \end{cases} , \forall\;  s \in \mS, \aa \in \mA.
%  \end{align}
%In this game, the rewards are the opposite of what the attacker wants.
%For example, with two learners and $\pidagi = \left(1, 1\right)$, the MLE $\hat \R(\cdot)$ from the original $\mD$ forms the reward matrix in Table~\ref{table:MLE_before}.

For simplicity, suppose that the same confidence half-width $\rhoR\left(\aa\right) = \rho < b$ is used for all $\aa$. Let $ \iota \in (0, b)$ be arbitrary. For this instance, to install $\pidag$ as the $\iota$-DSE, the attacker can flip all rewards in a way that is illustrated in Table~\ref{table:MLE_after}, inducing a cost as the upper bound in Theorem~\ref{thm:univeral_cost_bounds}. The situation is the same for $n\ge 2$ learners.
\begin{center} 
%\footnotesize
%\setlength\tabcolsep{2pt}
\begin{tabular}{|c|c|c|c|}
 \hline
 $\mA_{1} /  \mA_{2}$ &$1$ &$\dots$ &$2, ..., \left| \mA_{2} \right|$\\ \hline
 $1$ &$b, b$ &$\dots$ &$b, b \!-\! 2 \rho \!-\! \iota$\\ \hline
 $\vdots$ &$\vdots$ &$\vdots$&$\vdots$\\ \hline
 $2, ..., \left| \mA_{1} \right|$ &$b \!-\! 2 \rho \!-\! \iota, b$ &$\dots$ &$b \!-\! 2 \rho \!-\! \iota, b \!-\! 2 \rho \!-\! \iota$\\ \hline
 \end{tabular}
\captionof{table}{MLE $\hat \Rh(s,\cdot)$ after attack} \label{table:MLE_after}
\end{center}
Our instance-dependent lower bound,  presented later in Theorem~\ref{thm:wclb}, implies that any attack on this instance must have cost at least $NnA^{n-1}(2b + 2\rho + \iota)$. This lower bound matches the refined upper bound in the proof of Theorem~\ref{thm:ddbs}, implying the refined bounds are tight for this instance. Noticing that the universal bound in Theorem~\ref{thm:univeral_cost_bounds} only differs by an $O(A)$-factor implies it is nearly tight.

\subsection{Instance-Dependent Cost Bounds}

Next, we derive general bounds on the attack cost that depend on the structure of the underlying instance. Our strategy is to reduce the problem of bounding Markov Game costs to the easier problem of bounding Bandit Game costs. We begin by showing that the cost of poisoning a Markov Game dataset  can be bounded in terms of the cost of poisoning the datasets corresponding to its individual period games. 

\begin{thm}\label{thm:ddbs}
For any feasible attack instance $I$, we have that $C^*(I_H) \leq C^*(I)$ and,
%\[C^*(I_H) \leq C^*(I) \leq  \sum_{h = 1}^H C^*(I_h) + 2bn\overline{N}H|S|A.\]
\[C^*(I) \leq \sum_{h = 1}^H C^*(I_h) +  2bnH|\mS|\overline{N} + H^2 \overline{\rho} |\mS| nA^n \overline{N}\]
\end{thm} %sum of gaps at target action actually
\noindent Here we see the effect of the learner's uncertainty. If $\rhoR$ is small, then poisoning costs slightly more than poisoning each bandit instance independently. This is desirable since it allows the attacker to solve the much easier bandit instances instead of the full problem.

The lower bound is valid for all Markov Games, but it is weak in that it only uses the last period cost. However, this is the most general lower bound one can obtain without additional assumptions on the structure of the game. If we assume additional structure on the dataset, then the above lower bound can be extended beyond the last period, forcing a higher attack cost. 

\begin{lem}\label{lem:dslbmp}
%just need one possible transition to be uniform
Let $I$  be any feasible attack instance containing at least one uniform transition in $\CIPh$ for each period $h$, i.e., there is some $\Phh(s' \mid s,\aa)\in \CIPh$ with $\Phh(s' \mid s , \aa) = 1/|\mS|,\forall h, s',s,\aa$. Then, 
we have that
\[ C^*(I) \geq \sum_{h = 1}^H C^*(I_h). \]
\end{lem}
%$\textstyle\sum_{h = 1}^H \Nlh \textstyle\sum_{s \in \mS}\textstyle\sum_{i=1}^{n} \textstyle\sum_{\ani \in \mAni} \diih\left(s, \ani\right).$

\noindent In words, for these instances the optimal cost for poisoning is not too far off from the optimal cost of poisoning each period game independently. We note this is where the effects of $\rhoP$ show themselves. If the dataset is highly uncertain on the transitions, it becomes likely that a uniform transition exists in $\CIP$. Thus, a higher $\rhoP$ leads to a higher cost and effectively devolves the set of plausible games into a series of independent games. 

Now that we have the above relationships, we can focus on bounding the attack cost for bandit games. To be precise, we bound the cost of poisoning a period game instance $I_h$. To this end, we define $\iota$-dominance gaps.

\begin{df}(Dominance Gaps) \label{df:dgap} 
 For every $h \in \left[H\right], s \in \mS, i \in \left[n\right]$ and $\ani \in \mAni$, the $\iota$-dominance gap, $d_{i,h}^{\iota}(s,a_{-i})$, is defined as
 {
  %\vspace{-0.02in}
 \begin{align*}
 &\diih\left(s, \ani\right) := \\ &\Big[\displaystyle\max_{\ai \neq  \pidagih(s)} \left[\Rhih\big(s, (\ai, \ani)\big) + \rhoRh\big(s, (\ai, \ani)\big)\right] 
 \\ 
 &- \Rhih\Big(s, \big(\pidagih(s), \ani\big)\Big) + \rhoRh\Big(s, \big(\pidagih(s), \ani\big)\Big) + \iota\Big]_+
 \end{align*}}%
where $\hat R$ is the MLE w.r.t. the original dataset $\mD$.
\end{df}
%Define \epsilon_{i,h}^{\iota}(a_i, a_{-i}) = \rho^R(a_i, a_{-i}) + \rho^R(\pi^{\dagger}_i, a_{-i}) + \iota
%Then dominance gap has simpler form, $\diih\left(s, \ani\right) := \max_{\ai \neq  \pidagih(s)}[\Rhih\left(s, (\ai, \ani)\right) - \Rhih\left(s, \left(\pidagih(s), \ani\right)\right) \epsilon_{i,h}^{\iota}(a_i,a_{-i})]_+$
\noindent The dominance gaps measure the minimum amount by which the attacker would have to increase the reward for learner $i$ while others are playing $\ani$, so that the action $\pidagih\left(s\right)$ becomes $\iota$-dominant for learner $i$. We then consolidate all the dominance gaps for period $h$ into the variable $\Delta_h(\iota)$,
\begin{align*}
    \Delta_h(\iota) := \sum_{s \in S} \sum_{i = 1}^n \sum_{a_{-i}} \Big(d_{i,h}^{\iota}(s,a_{-i}) +\delta_{i,h}^{\iota}(s,a_{-i})\Big)
\end{align*}
%Replace with full definition?
%\begin{align*}
%    \Delta_h(\iota) &= \sum_{s \in S} \sum_{i = 1}^n \sum_{a_{-i}} \big(d_{i,h}^{\iota}(s,a_{-i}) \\
% &+\sum_{\substack{a_i \neq \pidag_i(s)\\ \hat{R}_{i,h}(s,(a_i,a_{-i})) > b - \epsilon_{i,h}(a_i,a_{-i}))}} \left(\hat{R}_{i,h}(s,a^i,a^{-i}) - b + \epsilon_{i,h}(a_i,a_{-i}) \right)\big)
%\end{align*}
%Or just define in the appendix to make the lemma as strong as possible?
Where $\delta_{i,h}^{\iota}(s,a_{-i})$ is a minor overflow term defined in the appendix. With all this machinery set up, we can give precise bounds on the minimal cost needed to attack a single period game.

\begin{lem}\label{lem:bbds}
The optimal attack cost for $I_h$ satisfies
\[\underline{N}_h\Delta_h(\iota)  \leq C^*(I_h) \leq \overline{N}_h \Delta_h(\iota).\]
\end{lem}
%+ n|\mS|A^{n}\iota if \Delta not changed.
\noindent Combining these bounds with Theorem~\ref{thm:ddbs} gives complete attack cost bounds for general Markov game instances.

%Although this attack cost is optimal in some games, it is generally loose. For example, if the target policy is already an $\iota$-MPDSE for the game, no poisoning is needed at all. One can potentially derive more refined upper bounds that depend on the structure of each stage game, which we leave as future work.

%Up to this point, things have gone very well for the attacker. Not only can the attacker always make the target policy $\pidag$ a Markov perfect dominant strategy equilibrium under mild conditions, but the attacker can conduct the poisoning efficiently using an LP. However, forcing $\pidag$ is not the only goal of the attacker. The attacker also desires a small attack cost. We give an attack cost lower bound that implies many games require a high cost to poison. 

%Now, we move to the lower bounds. With these definitions in hand, we present a lower bound on the attack cost given a full coverage dataset $\mD$.
%\begin{cor} \label{lem:dslb} 
% For any dataset $\mD$ with full coverage, and for any $\iota > 0$, the optimal attack cost with respect to $L^1$-norm is at least
 %\begin{align}
% $\NlH \sum_{s \in \mS}\textstyle\sum_{i=1}^{n} \textstyle\sum_{\ani \in \mAni} \diiH\left(s, \ani\right).$
 %\end{align}
 %\end{cor}
%We provide the proof in the appendix. %Appendix~\ref{proof: 5}.

%The proof can be found in Appendix~\ref{proof: 5}. %So, given additional structure on the original data, we see the attack must have higher cost. 

The lower bounds in both Lemma~\ref{lem:dslbmp} and Lemma~\ref{lem:bbds} expose an exponential dependency on $n$, the number of players, for some datasets $\mD$. These
instances essentially require the attacker to modify $\hat R_{i,h}(s,\aa)$ for every $\aa\in\mA$. A concrete instance can be constructed by taking the high cost dataset derived as the tight example before and extend it into a general Markov Game. We simply do this by giving the game several identical states and uniform transitions. In terms of the dataset, each episode consists of independent plays of the same normal-form game, possibly with a different state observed. For this dataset the $\iota$-dominance gap can be shown to be  $\diih\left(s, \ani\right) = 2 b + 2 \rho + \iota$. A direct application of Lemma~\ref{lem:dslbmp} gives the following explicit lower bound.
%Therefore, on this type of original dataset $\mD$ our lower bound becomes more transparent:
 
\begin{thm} \label{thm:wclb} 
 There exists a feasible attack instance $I$ for which it holds that 
 \[C^*(I) \geq \Nl H \left| \mS \right| n A^{n-1} \left(2 b + 2 \rho + \iota\right). \]
\end{thm}

Recall the attacker wants to assume little about the learners, and therefore chooses to install an $\iota$-MPDSE (instead of making stronger assumptions on the learners and installing a Nash equilibrium or a non-Markov perfect equilibrium).
On some datasets $\mD$, the exponential poisoning cost is the price the attacker pays for this flexibility.

\section{Conclusion}

We studied a security threat to offline MARL where an attacker can force learners into executing an
arbitrary Dominant Strategy Equilibrium by minimally poisoning historical data. We showed that the
attack problem can be formulated as a linear program, and provided analysis on the attack feasibility
and cost. This paper thus helps to raise awareness on the trustworthiness of multi-agent learning.
We encourage the community to study defense against such attacks, e.g. via robust statistics and
reinforcement learning.

\section*{Acknowledgements}

McMahan is supported in part by NSF grant 2023239.
Zhu is supported in part by NSF grants 1545481, 1704117, 1836978, 2023239, 2041428, 2202457, ARO MURI W911NF2110317, and AF CoE FA9550-18-1-0166.
Xie is partially supported by NSF grant 1955997 and JP Morgan Faculty Research Awards.
We also thank Yudong Chen for his useful comments and discussions.

\bibliography{aaai23}

\end{document}